\documentclass{article}

\usepackage[preprint]{neurips_2026}

\usepackage[utf8]{inputenc} 
\usepackage[T1]{fontenc}    
\usepackage{hyperref}       
\usepackage{url}            
\usepackage{booktabs}       
\usepackage{amsfonts}       
\usepackage{nicefrac}       
\usepackage{microtype}      
\usepackage{xcolor}         

\usepackage{amsmath}
\usepackage{wasysym}
\usepackage{enumitem}

\usepackage{booktabs}
\usepackage{multirow}
\usepackage{amsmath, amssymb, amsthm}

\usepackage{booktabs}
\usepackage{array}
\usepackage{makecell}
\usepackage{adjustbox}
\usepackage[table]{xcolor}
\usepackage{pifont}

\usepackage{amsmath,amssymb}
\usepackage{booktabs}
\usepackage[table]{xcolor}
\usepackage{graphicx}
\usepackage{array}

\usepackage{algorithm}
\usepackage{algpseudocode}
\algrenewcommand\algorithmicreturn{\textbf{Return:}}

\usepackage{booktabs}
\usepackage{multirow}
\usepackage[table]{xcolor}
\usepackage{graphicx}

\definecolor{lightred}{RGB}{252,235,235}
\definecolor{lightgreen}{RGB}{235,248,235}

\usepackage{booktabs}
\usepackage{multirow}
\usepackage[table]{xcolor}
\usepackage{colortbl}
\usepackage{array}

\definecolor{badred}{RGB}{248,228,228}
\definecolor{goodgreen}{RGB}{226,242,226}
\definecolor{warnorange}{RGB}{252,238,220}

\usepackage{booktabs}
\usepackage[table]{xcolor}
\usepackage{array}
\usepackage{tabularx}
\usepackage{amsmath}
\usepackage{amssymb}
\usepackage[most]{tcolorbox}

\definecolor{plugcoral}{RGB}{250,219,216}
\definecolor{plugcoraldeep}{RGB}{245,196,190}
\definecolor{tubegreen}{RGB}{198,224,180}
\definecolor{tubedeep}{RGB}{129,166,103}
\definecolor{refgray}{RGB}{233,236,239}
\definecolor{elbolavender}{RGB}{226,230,242}
\definecolor{gapamber}{RGB}{253,240,200}
\definecolor{baselinecream}{RGB}{248,243,228}
\definecolor{darkamber}{RGB}{156,113,21}
\definecolor{violationred}{RGB}{183,28,28}

\usepackage{titlesec}

\newtheorem{definition}{Definition}[section]

\definecolor{lightred}{RGB}{252,235,235}
\definecolor{lightgreen}{RGB}{235,248,235}

\definecolor{goodbg}{RGB}{232,245,233}
\definecolor{badbg}{RGB}{255,235,238}

\newcommand{\goodterm}[1]{%
  \begingroup
  \setlength{\fboxsep}{1pt}%
  \colorbox{goodbg}{\ensuremath{\displaystyle #1}}%
  \endgroup
}
\newcommand{\badterm}[1]{%
  \begingroup
  \setlength{\fboxsep}{1pt}%
  \colorbox{badbg}{\ensuremath{\displaystyle #1}}%
  \endgroup
}

\definecolor{goodbg}{RGB}{231,245,233}
\definecolor{warnbg}{RGB}{255,249,219}
\definecolor{badbg}{RGB}{255,235,238}
\definecolor{goodfg}{RGB}{27,94,32}
\definecolor{badfg}{RGB}{183,28,28}

\newcommand{\cmark}{\ding{51}}
\newcommand{\xmark}{\ding{55}}

\usepackage{booktabs}
\usepackage{multirow}
\usepackage{array}
\usepackage{colortbl}
\usepackage[table]{xcolor}
\usepackage{makecell}

\definecolor{tubegreen}{RGB}{226,239,218}
\definecolor{plugred}{RGB}{248,228,228}
\definecolor{refgray}{RGB}{240,240,240}
\definecolor{ubblue}{RGB}{222,232,246}

\newcolumntype{P}[1]{>{\raggedright\arraybackslash}p{#1}}

\usepackage{subcaption}
\usepackage[most]{tcolorbox}
\usepackage{wrapfig}
\allowdisplaybreaks
\usepackage{tikz}
\usetikzlibrary{arrows.meta,calc}

\newtheorem{proposition}{Proposition}[section]

\title{TUBE: Tangent Upper Bound on Evidence \\ for Discrete Diffusion Language Models}

\author{%
  \textbf{Arseny Ivanov}$^{1,2,3}$, \textbf{Sergei Kholkin}$^{2}$, \textbf{Vladislav Gromadskii}$^{2}$,\\
  \textbf{Grigoriy Ksenofontov}$^{2,4}$, \textbf{Ivan Oseledets}$^{1,2}$, \textbf{Alexander Korotin}$^{2,1}$%
}

\begin{document}

\maketitle

\renewcommand{\thefootnote}{}%
\footnotetext{%
  $^{1}$\,AXXX, Russia; \quad
  $^{2}$\,Applied AI Institute, Moscow, Russia; \quad
  $^{3}$\,HSE University, Moscow, Russia. \quad
  $^{4}$\,MIRAI, Russia.\\
  $^{*}$\,Correspondence to: Arseny Ivanov \texttt{<a5r5s5e5n5y@gmail.com>}, Alexander Korotin \texttt{<iamalexkorotin@gmail.com>}%
}%
\renewcommand{\thefootnote}{\arabic{footnote}}%
\setcounter{footnote}{0}

\begin{abstract}
    Log-likelihood is a standard metric for evaluating generative models. Unfortunately, in contrast to autoregressive models (ARMs), discrete diffusion models generally do not admit exact computation of this quantity. Existing evaluations, therefore, rely on the evidence lower bound (ELBO), leaving unclear how much higher the true value may be. We address this by introducing the \textbf{Tangent Upper Bound on Evidence} (\textbf{TUBE}), a variational upper bound on log-likelihood that admits an unbiased Monte Carlo estimator. Our TUBE extends across latent-variable models, including masked diffusion models (MDMs), any-order ARMs (AO-ARMs), and block variants of both. Applied to block MDMs and block AO-ARMs, TUBE reveals our key empirical finding that these models lie strictly below the exact ARM baseline, showing that ARMs still dominate in likelihood.
\end{abstract}

\begin{figure}[h]        
  \centering         
  \includegraphics[width=1.0\textwidth]{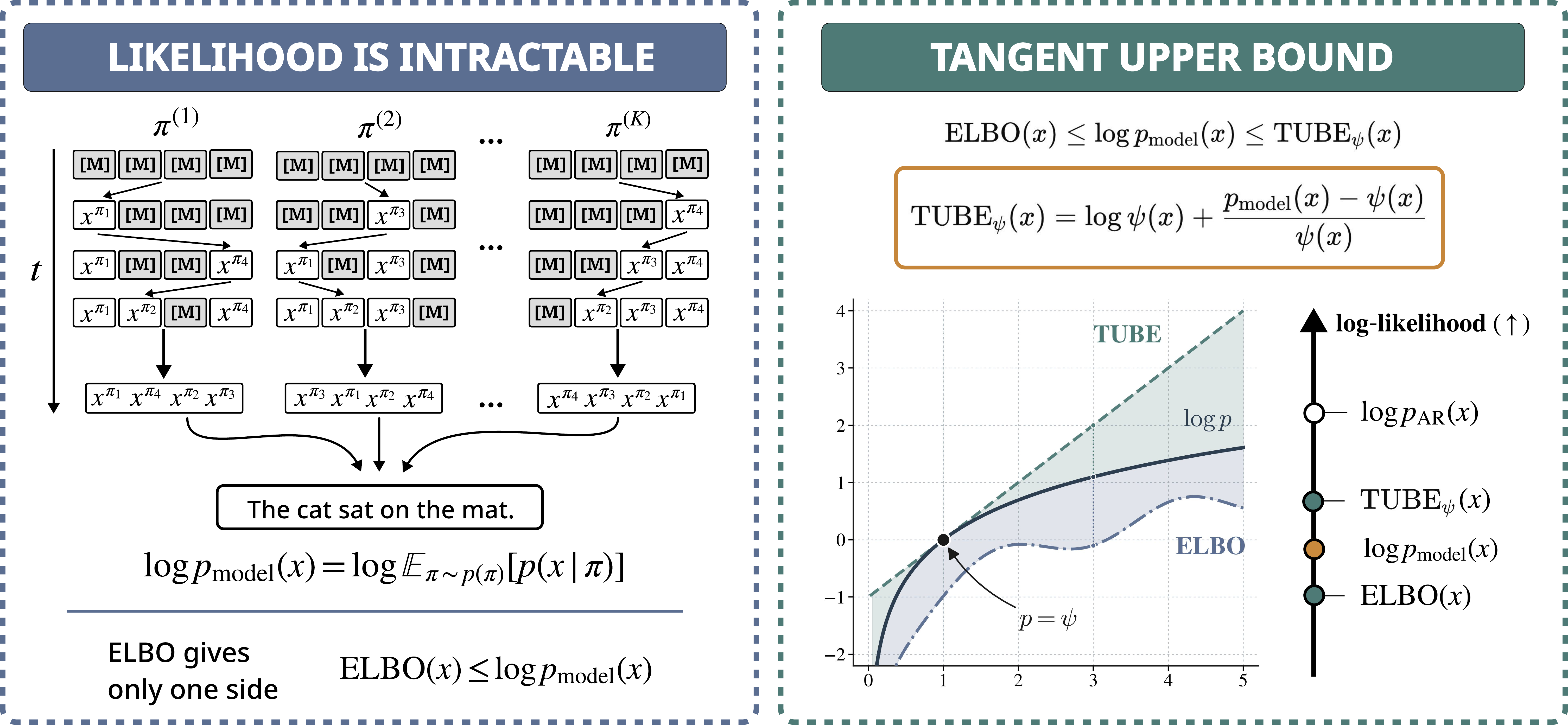}
  \caption{\textbf{A tight tangent upper bound on $\log p_{\mathrm{model}}(x)$.} Our TUBE bounds the intractable marginal from above using a tractable surrogate $\psi$, with equality at $\psi=p_{\mathrm{model}}$.}
  \label{fig:tube}   
\end{figure}

\section{Introduction}
\emph{Autoregressive models} (ARMs) remain a central paradigm in language modeling, supported by strong empirical scaling behavior in large-scale settings \citep{kaplan2020scaling,hoffmann2022training}. However, this efficiency arises from imposing a fixed autoregressive decomposition, which makes generation inherently order-dependent. At the same time, the best ordering typically depends on the domain and task, and learned or alternative orderings can outperform standard fixed choices in practice \citep{li2021discovering,wang2025learningorder}. This has motivated a growing line of work on \emph{any-order autoregressive models} (AO-ARMs) \citep{uria2014deep,shih2022training} and related \emph{masked diffusion models} (MDMs) \citep{austin2021structured,shi2024simplified,sahoo2024mdlm}, which both replace a single fixed decomposition with a probabilistic family of generation orderings. 

Despite their different theoretical constructions, AO-ARMs and MDMs both generalize beyond left-to-right order by repeatedly selecting token positions at random and filling in the corresponding values. MDMs further extend this procedure by allowing multiple tokens to be generated in a single step, enabling faster inference. Beyond these fully random ordering schemes, blockwise generation provides a mixed alternative \citep{arriola2025block}, using a fixed autoregressive ordering over blocks and a random ordering within each block. This retains block-level autoregressive structure, enabling techniques such as KV caching. Together, these features make AO-ARMs and MDMs viable alternatives to strictly autoregressive generation. More broadly, both have been applied beyond standard left-to-right language modeling to other discrete domains, including images \citep{pang2025randar,austin2021structured}, graphs \citep{kelvinius2024discriminator,seo2025learning}, molecular sequences \citep{lee2025genmol}, and vector-quantized image representations \citep{gu2022vector}. In language modeling, recent large-scale MDMs have been shown to be competitive in the generation quality to ARMs \citep{nie2025llada,bie2025llada2,ye2025dream7b,monsefi2026fsdfm}.


At the same time, principled \textbf{evaluation} of such models calls for \textbf{log-likelihood}, which remains the canonical measure of distributional fit. While ARM enables exact log-likelihood evaluation, for AO-ARM and MDM, on the other hand, this quantity is typically intractable. It is instead replaced by the \emph{evidence lower bound} (ELBO), or related approximations \citep{sahoo2024mdlm,haxholli2025perplexity,jeon2025itdd}. However, the ELBO does not indicate how much higher the true log-likelihood may be. As a result, ELBO-based evaluation alone cannot reliably assess the likelihood of AO-ARMs and MDMs, nor support rigorous comparison with ARM baselines.





Recent work has made this issue increasingly explicit both for MDMs and for broader classes of latent-variable models. For MDMs, exact log-likelihood evaluation is available only in special cases, such as deterministic unmasking \citep{turok2026duel}, while recent work has also proposed upper-bound estimator \citep{wang2026spg}. In parallel, the variational-inference literature has developed upper-bound approaches for estimating log-likelihood in general latent-variable models \citep{dieng2017chivi,struski2023bounding}. In practice, however, exact methods apply only in narrow fixed-order settings, and upper-bound estimators are hard to use reliably because they become biased when a nonlinear function is applied after Monte Carlo approximation. This leaves open the problem of reliable finite-sample log-likelihood estimation for AO-ARMs and MDMs, which we address in this paper. Our \textbf{contributions} are as follows:
\begin{itemize}[leftmargin=*]
    \item \textbf{Method.} We propose the \emph{Tangent Upper Bound on Evidence} (TUBE), a variational pointwise upper bound on the log-likelihood with a tractable surrogate and an unbiased Monte Carlo estimator (\S\ref{sec:method}). Together with the ELBO, TUBE yields a two-sided localization of the log-likelihood.
    \item \textbf{Analysis.} We evaluate pretrained block AO-ARMs and MDMs and identify a clear empirical likelihood gap relative to standard ARM baselines (\S\ref{sec:experiments}).
\end{itemize}

\textbf{Notation.} We write $\mathcal{V}\!=\!\{1,\ldots,V\}^{L}$ for the space of length-$L$ token sequences and $x\!=\!(x^1,\ldots,x^L)\!\in\!\mathcal{V}$ for a data sample. We use $p$ to denote distributions, with subscripts indicating the corresponding distribution when needed, e.g., $p_{\mathrm{model}}$. The generation ordering is denoted by $\pi$. By $\mathcal{S}$ we denote single-token orderings, where each step reveals one position, and $\mathcal{G}$ denotes grouped orderings, where a step may reveal multiple positions. For an index set $\pi_t$ at step $t\in\{1,\dots,T\}$, $x^{\pi_t}$ denotes the tokens of $x$ at positions in $\pi_t$, and $\pi_{<t}$ denotes the positions revealed before that step.

\section{Background}
\label{sec:background}
In this section, we first recall ARMs, including their definition and training procedure (\S\ref{sec:arm}). We then generalize this view to AO-ARMs (\S\ref{sec:ao-arm}) and relate them to MDMs (\S\ref{sec:mdm}). Finally, we introduce block variants of the latter two (\S\ref{sec:bm}) and discuss the challenges of likelihood evaluation (\S\ref{sec:problematics}).

\subsection{Autoregressive models (ARM)}
\label{sec:arm}
We consider the problem of generating a sequence $x=(x^1,\dots,x^L)\in\mathcal{V}$ of length $L$, where each $x^l$ denotes a token. ARMs $p_{\text{ARM}}(x)$ model this sequence by generating one token at each step $t\in\{1,\dots,L\}$ according to a fixed generation order $\pi\in\mathcal{S}$:
\vspace{-2mm}
\begin{equation}
    p_{\text{ARM}}(x)
    = p(x^{\pi_1}, \dots, x^{\pi_L})
    = p(x^{\pi_1}) \dots p(x^{\pi_L} | x^{\pi_{L-1}}, \dots, x^{\pi_1}) = \prod_{t=1}^{L} p(x^{\pi_t} | x^{\pi_{<t}}),
    \label{eq:ar}
\end{equation}
where $\pi_t$ denotes the single position generated at step $t$, so the full sequence $x$ is generated in $L$ steps.

For ARMs, log-likelihood can be evaluated directly due to the fixed decomposition in \eqref{eq:ar}:
\begin{align}
    \log p_{\text{ARM}}(x)
    = \sum_{t=1}^{L} \log p(x^{\pi_t} | x^{\pi_{<t}}).
\end{align}
Common choices of $\pi$ include left-to-right orderings for text \citep{radford2018improving} and raster-scan orderings for images \citep{chen2020generative}. Fixed orderings make training and log-likelihood evaluation tractable. However, they force the model to generate variables in one prescribed order, even when other orders may be equally natural or more suitable for the data.

\begin{figure}[t]    
      \centering         
      \includegraphics[width=1.0\textwidth]{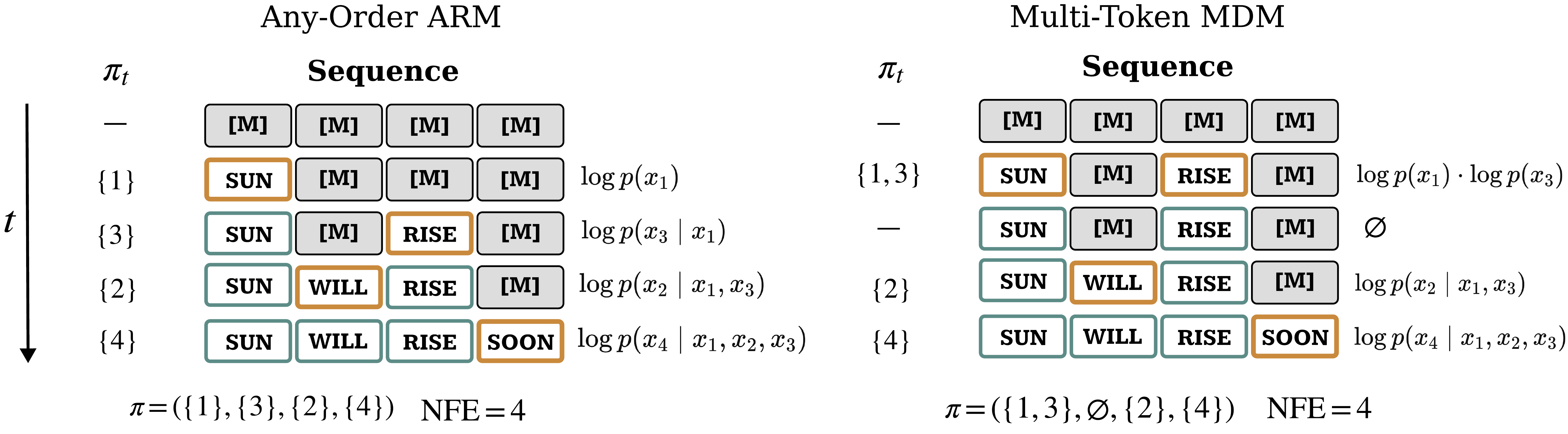}
      \caption{\textbf{AO-ARM and MDM.} Both define $p_{\mathrm{model}}(x)$ via a latent generation ordering $\pi$. AO-ARM reveals one token per step ($\pi\in\mathcal{S}$). MDM allows zero, one, or many tokens per step ($\pi\in\mathcal{G}$).}
      \label{fig:ao_ar_vs_md}
\end{figure} 

\subsection{Any-order autoregressive models (AO-ARM)}
\label{sec:ao-arm}
To reduce dependence on a fixed order, AO-ARMs \citep{uria2014deep,shih2022training} model the data distribution by averaging over autoregressive decompositions induced by different orderings:
\begin{equation}
    p_{\text{AO-ARM}}(x)
    = \mathbb{E}_{\pi \sim p(\pi)} \left[ p(x | \pi) \right]
    = \mathbb{E}_{\pi \sim p(\pi)} \left[ \prod_{t=1}^{L} p(x^{\pi_t} | x^{\pi_{<t}}) \right],
    \label{eq:ao_arm}
\end{equation}
where $p(\pi)$ is typically a uniform distribution over orderings $\pi\in\mathcal{S}$, and generation takes $L$ steps.

From this perspective, AO-ARMs can be viewed as latent-variable models in which the ordering $\pi$ is latent, thereby increasing modeling flexibility. Unfortunately, exact log-likelihood computation is generally infeasible due to the large size of $\mathcal{S}$. One may therefore approximate the expectation in \eqref{eq:ao_arm} using Monte Carlo, but placing this approximation inside the logarithm yields a biased estimator of the log-likelihood. Thus, in previous works \citep{uria2014deep,hoogeboom2022autoregressive}, log-likelihood estimation is typically based on ELBO:
\begin{equation}
    \log p_{\text{AO-ARM}}(x)
    = \log \mathbb{E}_{\pi \sim p(\pi)} \left[ \prod_{t=1}^{L} p(x^{\pi_t} | x^{\pi_{<t}}) \right] \geq \mathbb{E}_{\pi \sim p(\pi)} \left[ \sum_{t=1}^{L} \log p(x^{\pi_t} | x^{\pi_{<t}}) \right].
    \label{eq:ao_arm_elbo}
\end{equation}
This latent-order perspective also provides a natural bridge to discrete diffusion models.

\subsection{Masked diffusion models (MDM)}
\label{sec:mdm}
In MDMs \citep{austin2021structured,shi2024simplified,sahoo2024mdlm}, the latent variable can be interpreted as an unmasking trajectory, or equivalently as a grouped ordering; see Figure~\ref{fig:ao_ar_vs_md}. We write an unmasking trajectory as $\pi=(\pi_1,\dots,\pi_T)\in\mathcal{G}$, where $\mathcal{G}$ denotes the set of \emph{grouped orderings}. At each generation step $t\in\{1,\dots,T\}$, the set $\pi_t\subseteq\{1,\dots,L\}$ contains the positions revealed at that step. The sets $\pi_t$ are pairwise disjoint, may be empty, and cover all positions, i.e., $\bigsqcup_{t=1}^{T}\pi_t=\{1,\dots,L\}$. Thus, a step may reveal zero, one, or multiple positions, enabling parallel token generation. To make such parallel updates tractable, the joint distribution over the revealed tokens is typically factorized over positions, $p(x^{\pi_t} | x^{\pi_{<t}}, \pi) = \prod_{d\in \pi_t}p(x^d | x^{\pi_{<t}}, \pi)$. With this parameterization, MDMs mirror AO-ARMs by averaging over grouped generation orderings:
\begin{equation}
    p_{\mathrm{MDM}}(x)
    =
    \mathbb{E}_{\pi \sim p(\pi)}
    \left[
        p(x | \pi)
    \right]
    =
    \mathbb{E}_{\pi\sim p(\pi)}
    \left[
        \prod_{t=1}^{T}
        \prod_{d\in \pi_t}
        p(x^d | x^{\pi_{<t}}, \pi)
    \right],
    \label{eq:mdm}
\end{equation}
where $p(\pi)$ is the distribution over unmasking trajectories induced by the chosen masking process. Generation proceeds in $T$ steps, which may be smaller or larger than $L$.

Notably, the grouped-ordering view also clarifies the relation to AO-ARMs in continuous time. In models such as \citep{sahoo2024mdlm}, each position has a continuous generation time $t\in[0,T]$, so different positions are revealed at the same time with probability zero. After sorting the reveal events and removing empty intervals, the trajectory reduces to single-position groups. As a result, generation proceeds in an AO-ARM-like manner, revealing one token at a time according to a random ordering $\pi\in\mathcal{S}$ induced by the masking process.

Regarding log-likelihood estimation, in continuous time one can use the standard MDM ELBO, e.g., \citep[Eq.~2.6]{ou2025absorbing}, which is notably equivalent to the AO-ARM ELBO in \eqref{eq:ao_arm_elbo} \citep[Thm.~2]{ou2025absorbing}. For discretized-time MDMs, however, the relation to AO-ARMs yields a similar ELBO, but over grouped trajectories $\pi\in\mathcal{G}$:
\begin{equation}
    \log p_{\mathrm{MDM}}(x)
    \!=
    \log
    \mathbb{E}_{\pi\sim p(\pi)}\!
    \left[
        \prod_{t=1}^{T}
        \prod_{d\in \pi_t}\!
        p(x^d | x^{\pi_{<t}}\!, \pi)
    \right]
    \!\geq
    \mathbb{E}_{\pi\sim p(\pi)}\!
    \left[
        \sum_{t=1}^{T}
        \sum_{d\in \pi_t}
        \log p(x^d | x^{\pi_{<t}}\!, \pi)
    \right]\!.
    \label{eq:mdm_elbo}
\end{equation}

\subsection{Block models (BM)}
\label{sec:bm}
BMs \citep{arriola2025block} provide another way to combine the tractability of ARMs with the flexibility of AO-ARMs and MDMs. Rather than defining a single ordering over all positions, they split the sequence into $B$ blocks of size $L'$, denoted by disjoint position sets $\mathcal{B}_1,\dots,\mathcal{B}_B$ with $\bigsqcup_{b=1}^{B}\mathcal{B}_b=\{1,\dots,L\}$. This gives the following block-level decomposition:
\begin{equation}
    p_{\mathrm{BM}}(x)
    =
    \prod_{b=1}^{B}
    p(x^{\mathcal{B}_b} | x^{\mathcal{B}_{<b}}),
    \qquad \text{where }
    p(x^{\mathcal{B}_b} | x^{\mathcal{B}_{<b}})
    \text{ is parameterized as in } \eqref{eq:ao_arm} \text{ or } \eqref{eq:mdm}.
    \label{eq:block_decomp}
\end{equation}
As a result, the block factorization preserves the autoregressive structure needed for techniques such as KV caching, while the within-block conditionals can still use random-order or parallel generation.

In both parameterizations, the decomposition in \eqref{eq:block_decomp} is restricted to positions inside $\mathcal{B}^b$ and conditioned on the previously generated blocks $x^{\mathcal{B}_{<b}}$. Thus, log-likelihood evaluation reduces to a sum of block conditional log-likelihoods. However, the AO-ARM formulation in \eqref{eq:ao_arm} and the MDM formulation in \eqref{eq:mdm} still require averaging within each block over local orderings $\pi\in\mathcal{S}$ or grouped orderings $\pi\in\mathcal{G}$, making exact evaluation intractable and requiring the ELBOs from \eqref{eq:ao_arm_elbo} and \eqref{eq:mdm_elbo}, respectively.

\subsection{Likelihood evaluation challenges}
\label{sec:problematics}

Likelihood-based comparison of ARMs, AO-ARMs, and MDMs involves two separate choices: the trained model $p_{\text{model}}$ and the estimator used to evaluate its likelihood. For the model, notably, a single network trained with either an AO-ARM or MDM ELBO can be reused under different generation orderings. This is possible because these formulations parameterize the same family of conditional distributions, $p_{\text{model}}(x^d \mid x^{\pi_{<t}}, \pi_{<t})$ for $d\in\pi_t$, with ARMs recovered as the fixed-order special case. Moreover, their training objectives have equivalent minimizers under the connection established by \citep{ou2025absorbing}. For estimator, on the other hand, the ELBO in \eqref{eq:ao_arm_elbo} and \eqref{eq:mdm_elbo} do not localize the true log-likelihood of AO-ARMs and MDMs, and therefore cannot support rigorous comparison with ARMs on their own. We therefore introduce a complementary upper bound that, together with the ELBO, localizes the log-likelihood of AO-ARMs and MDMs.

\section{Method}
\label{sec:method}
We construct our TUBE, a variational upper bound on $\log p(x)$ that, combined with ELBO, localizes the likelihood from both sides. We first state the bound itself (\S\ref{sec:tube_agnostic}). Then take a look at its practical implementation for AO-ARM and MDM models in (\S\ref{sec:methods_tube_for_aoarm_mdm}).



\subsection{Tangent Upper Bound on Evidence (TUBE)}
\label{sec:tube_agnostic}

Our method of constructing the upper bound on $\log p(x)$ rests on the concavity of $\log a$ function in $a$ and the tangent majorization of $\log a$ by the taylor decomposition at point $b$:
\begin{equation}
    \log a \leq \log b + \frac{a - b}{b}, \quad \forall\, a, b > 0.
\end{equation}
The upper bound has $b$ as the variational variable and allows us to construct a linear bound for $\log p(x)$ treating the upper bound at each $a=p(x)$ separately.


\begin{tcolorbox}[
  colback=gray!8,
  colframe=gray!60!black,
  boxrule=0.6pt,
  arc=2pt,
  left=6pt,
  right=6pt,
  top=4pt,
  bottom=4pt
]
\begin{definition}[Tangent Upper Bound on Evidence]
\label{thm:tube}
Consider positive functions $p(x)$ and $\psi(x)$, where $p(x), \psi(x) > 0$ for every $x\in\mathcal{X}$. The following variational bound holds:
\begin{equation}
\log p(x)\;\le\;\mathrm{TUBE}_{\psi}(x)
\;:=\;\log\psi(x)\;+\;\frac{p(x)-\psi(x)}{\psi(x)},
\label{eq:tube}
\end{equation}
where the function $\psi(x)$ is a auxiliary function, and equality holds iff $\psi(x)=p(x)$.
\end{definition}
\end{tcolorbox}


This bound being based on the logarithm linearization allows for linearization of $p(x)$ term, which in context of latent variable models of type $p(x) = \mathbb{E}_\pi[ p(x|\pi)]$, see \S\ref{sec:background},  allows for unbiased Monte Carlo estimation of the bound, see \S\ref{sec:methods_tube_for_aoarm_mdm} below. The introduced function $\psi(x)$, which we call \textbf{surrogate}, can be any positive function and the tightness of TUBE depends on the closeness of $\psi(x)$ to $p(x)$.

Since TUBE is an upper bound, combining it with any lower bound, such as
ELBO or its multi-sample extension $\mathrm{ELBO}_K$, yields a population two-sided
localization of $\log p(x)$ for any $\psi$: 

\begin{tcolorbox}[
  enhanced,
  colback=white,
  colframe=black,
  boxrule=0.6pt,
  arc=4pt,
  left=8pt,
  right=8pt,
  top=12pt,
  bottom=8pt,
  attach boxed title to top left={
    xshift=10pt,
    yshift=-8pt
  },
  boxed title style={
    enhanced,
    colback=white,
    colframe=black,
    boxrule=0.6pt,
    arc=3pt,
    left=4pt,
    right=4pt,
    top=2pt,
    bottom=2pt,
  },
  title={\textbf{Two-sided localization of log-likelihood}},
  fonttitle=\normalsize\bfseries,
  coltitle=black
]
\begin{equation}
\mathrm{ELBO}(x)\;\le\;\log p(x)\;\le\;\mathrm{TUBE}_{\psi}(x).
\label{eq:sandwich_pop}
\end{equation}
\end{tcolorbox}


\subsection{TUBE estimation for AO-ARM and MDM}\label{sec:methods_tube_for_aoarm_mdm}

\textbf{\underline{Monte Carlo estimation.}} While TUBE can applied to any latent variable model $p(x) =\mathbb{E}_z[p(x|z)]$, such as VAE \cite{kingma2013autoencoding} or DDPM \cite{ho2020denoising}, our work focuses on the application of TUBE for AO-ARM \cite{hoogeboom2022autoregressive} and MDM \cite{sahoo2024mdlm} models. Let us consider the model $p_{\rm model}=\mathbb{E}_\pi[ p_{\rm model}(x|\pi)]$ and its Monte Carlo approximated likelihood, then we can finally construct an \textbf{unbiased estimator} for the TUBE$_\psi$:

\begin{equation}
    \widehat{\mathrm{TUBE}_{\psi,K}}(x)\!:=\!\log\psi(x)+\frac{\widehat{p}_{\mathrm{model},K}(x)-\psi(x)}{\psi(x)},\quad \widehat{p}_{\mathrm{model},K}(x):=\frac{1}{K}\sum_{k=1}^{K} p_{\rm model}(x| \pi^{(k)}),
    \label{eq:tube_mc}
\end{equation}
where $\pi^{(1)},\dots,\pi^{(K)}\overset{\mathrm{i.i.d.}}{\sim}p(\pi)$. This straightforward possibility to construct the unbiased estimator is one of the core properties of our bound. This is a contrast with other upper bounds like CUBO \cite{dieng2017chivi} or TVO \cite{masrani2019tvo} where the estimation is biased, see Table~\ref{tab:estimator_zoo} and discussion in \S\ref{sec:related}. In detail, for AO-ARM case the Monte Carlo $\widehat{p}_{\mathrm{model},K}(x)$ takes the form: 
\begin{equation}
    \widehat{p}_{\mathrm{model},K}(x) =\frac{1}{K} \sum_{k=1}^K \left[ \prod_{t=1}^{L} p_{\rm model}(x^{\pi^{(k)}_t} | x^{\pi_{<t}}) \right], \quad \pi^{(1)},\dots,\pi^{(K)}\overset{\mathrm{i.i.d.}} \sim p(\pi), \quad \pi^{(k)} \in \mathcal{S}
\end{equation}
and in the case of MDM with $T$ steps:
\begin{equation}
    \widehat{p}_{\mathrm{model},K}(x) = \frac{1}{K}\sum_{k=1}^K
    \left[
        \prod_{t=1}^{T}
        \prod_{d\in \pi_t}
        p_{\rm model}(x^d | x^{\pi_{<t}}, \pi^{(k)})
    \right], \quad \pi^{(1)},\dots,\pi^{(K)}\overset{\mathrm{i.i.d.}} \sim p(\pi), \pi^{(k)} \in \mathcal{G}.
\end{equation}
The Block Models extensions simply follows from \eqref{eq:block_decomp}. 



\textbf{\underline{Choice of surrogate $\psi$.}}
\label{sec:choice_surrogate_method}
Since the tightness of TUBE depends on the closeness of surrogate $\psi(x)$ to the $p(x)$, the choice of it is important. Here we propose two versions of its construction.

\textbf{(1) Exact likelihood surrogate model.} One option is to take the model which has the similar density to $p_{\rm model}(x)$, but allows for fast and exact likelihood computation, e.g., ARM. Since $p_{\rm model}$ is trained on data $p_{\rm data}$, than it is natural to take ARM trained on the same data ($p_{\rm ARM}\approx p_{\rm data}\approx p_{\rm model}$):
\begin{equation}
    \psi_{ARM}(x) = p_{\text{ARM}}(x)
    = \prod_{t=1}^{L}  p_\psi(x^{\pi_t} | x^{\pi_{<t}}).
\end{equation}
Furthermore, to bring the surrogate $\psi_{ARM}$ model closer to the $p_{\rm model}$ one may finetune the $\psi_{ARM}$ using the synthetic data generated by $p_{\rm model}$, which can be done be regular likelihood maximization procedure \cite{radford2018improving} or more complex procedures \cite{ouyang2022instructgpt}.

\textbf{(2) Self-surrogate.} Another option, is to suppose that conditional on arbitrary order $\pi$ likelihood $p_{\rm model}(x|\pi)$ is close to full model likelihood, i.e., $p_{\rm model}(x|\pi) \approx p_{\rm model}(x)$, which is in fact true at the model optimum \cite{ou2025absorbing}. Furthermore, we can take several orders $\pi^{(m)}$ and average $p_{\rm model}(x|\pi^{(m)})$ for them, which on practice should be even closer to the $p_{\rm model}(x)$. We consider:
\begin{equation}\label{eq:psi_mc}
   \psi_M(x) = \frac{1}{M}\sum_{m=1}^M  p_{\rm model}(x|\hat{\pi}^{(m)})
\end{equation}
where $\hat{\pi}^{(1)},\dots,\hat{\pi}^{(M)}$ are some chosen orders. Note that both $\widehat{p}_{\mathrm{model},K}(x) $ and $\psi_M(x)$ utilize orders $\pi^{(k)}$ and $\hat{\pi}^{(m)}$. If one samples the orders for $\psi_{M}$ at random, e.g., from $p(\pi)$, they should be \textit{independent} from orders for $\widehat{p}_{\mathrm{model},K}(x)$ to avoid correlation bias in $\mathrm{TUBE}_{\psi,K}(x)$.



\paragraph{The computation of TUBE.}
The practical procedure for two-sided likelihood localization consists of sampling latent variables $\pi^{(k)}$, evaluating the target model $p_{\rm model}(x \mid \pi^{(k)})$, computing the surrogate $\psi(x)$, and then evaluating~\eqref{eq:tube_mc}. One can see the \underline{detailed algorithm} in Appendix~\ref{app:add_method}. The computational complexity is then determined by two components: \textit{(1)} evaluation of the likelihood model $p_{\rm model}$ and \textit{(2)} evaluation of the surrogate $\psi$. Estimating complexity of $\widehat{p}_{\text{model},K}$ scales linearly with the number of Monte Carlo samples $K$, i.e., $\mathcal{O}(K)$, since it requires $K$ evaluations of full $p_{\rm model}(x \mid \pi^{(k)})$. The cost of $\psi$ depends on its form: for $\psi_{\rm ARM}$, it is a single forward pass through the AR model, i.e., $\mathcal{O}(1)$, while for $\psi_M(x)$ the cost is similar to $\widehat{p}_{\text{model},K}$ and linear in $M$, i.e., $\mathcal{O}(M)$.

\section{Related work}
\label{sec:related}
\textbf{External-evaluator metrics}. A parallel line of work bypasses $\log p_{\mathrm{model}}(x)$ and scores samples under a separate pretrained reference model. \emph{Generative perplexity (gen-PPL)}~\citep{lou2024sedd,sahoo2024mdlm} reports the perplexity of samples under GPT-2 Large. However, it can be lowered simply by reducing sample diversity~\citep{zheng2024masked}, motivating the de facto practice of reporting it alongside sample entropy which complicated comparisons. \emph{MAUVE}~\citep{pillutla2021mauve} replaces gen-PPL's token-level scoring with an embedding-space divergence and inherits the encoder's coverage biases. \citep{pynadath2026genfrontiers} formalizes the perplexity-entropy pairing via a KL decomposition and proposes a frontier-based comparison, while \citep{salimans2026dmmd} introduces a gradient-based alternative to gen-PPL that vanishes when samples match the training distribution. All of these score the model's samples through a chosen reference rather than evaluating $\log p_{\mathrm{model}}(x)$ on real data, so they cannot be compared to ARM's exact log-likelihood on the same axis.

\paragraph{Log-likelihood approximation for MDMs.}
Evaluating $\log p_{\mathrm{model}}(x)$ for a pretrained MDM is intractable, so most works report single-sample ELBO~\citep{sahoo2024mdlm,haxholli2025perplexity} or its multi-sample tightening \textbf{ELBO}$_K$, also known as I-WAE ~\citep{burda2016iwae}, see Table \ref{tab:estimator_zoo}. Two recent works go further: DUEL~\citep{turok2026duel} computes $\log p_{\mathrm{model}}(x)$ exactly, but only under deterministic unmasking, restricting the sampler family, while SPG~\citep{wang2026spg} introduces a sandwich estimate of $\log p_{\mathrm{model}}(x)$ between ELBO and R\'enyi upper bound~\citep{li2016renyi}.

\paragraph{Upper bound estimators on log-likelihood.}
Here we highlight three established families of generic upper-bound estimators on $\log p_{\mathrm{model}}(x)$ for latent variable models that are applicable to MDMs and most related to our work: \textbf{CUBO$_\beta$, TVO$_L$, IS-VG-B}, see  Table~\ref{tab:estimator_zoo} and Appendix~\ref{app:estimator_comparison} for \underline{summary}. 
\begin{itemize}[leftmargin=*]
    \item \underline{\textbf{CUBO$_\beta$ / R\'enyi}~\citep{li2016renyi,dieng2017chivi,wang2026spg}} take the form $\tfrac{1}{\beta}\log\mathbb{E}_\pi[p_{\mathrm{model}}(x|\pi)^\beta]$ for $\beta\!\ge\!1$, an upper bound on $\log p_{\mathrm{model}}(x)$; exact at $\beta\!=\!1$. Its empirical estimator (see  Table~\ref{tab:estimator_zoo}) applies an outer $\log$ to a Monte Carlo mean, so Jensen acts \textit{against} the upper bound at finite $K$, adding bias which may remove the upper bound guarantee for expectation.
    \item
\underline{\textbf{TVO}~\citep{masrani2019tvo}} is a right Riemann sum of the thermodynamic identity along the geometric path $q_\beta(\pi|x)\propto p_{\mathrm{model}}(x|\pi)^\beta$, discretized at $\Lambda$ grid points $\beta_\lambda = \lambda/\Lambda$, with self-normalized weights biased at finite $K$ and an $\mathcal{O}(1/\Lambda)$ discretization residual (Table~\ref{tab:estimator_zoo}, Appendix~\ref{app:tvo}).
\item \underline{\textbf{IS-VG-B}~\citep{struski2023bounding}} adds an independent log-ratio correction to $\mathrm{ELBO}_s$ governed by two hyperparameters ($n_p \geq 2$ pairs of $s$-sample MC estimators, total budget $K = 2\,s\,n_p$). The bound is rigorous in the population, but the empirical estimator is biased and not bound-preserving, since both log terms apply Jensen-on-log to MC averages (see Table~\ref{tab:estimator_zoo} and Appendix~\ref{app:isvgb}).
\end{itemize}
    
In comparison to all these bounds, our TUBE admits unbiased bound-preserving estimator.

\begin{table}[htb]
\centering
\small
\caption{Estimators for various bounds on log-likelihood for latent-variable models. We say that the empirical estimator is \textbf{bound-preserving} if its \textit{expectation} provably preserves the bound direction. Bound-breaking terms are highlighted in red; our TUBE's estimator is bound-preserving.}
\label{tab:estimator_zoo}
\renewcommand{\arraystretch}{1.1}
\setlength{\tabcolsep}{5pt}
\resizebox{\textwidth}{!}{%
\begin{tabular}{@{} l c c c @{}}
\toprule
\textbf{Bound name} & \textbf{Bound type} & \textbf{Empirical estimator} & \textbf{Bound-preserving?} \\
\midrule
$\mathrm{ELBO}$
  & Lower
  & $\log p_{\mathrm{model}}(x| \pi^{(1)})$
  & \cellcolor{goodbg}\textcolor{goodfg}{\cmark} \\
$\mathrm{ELBO}_K$
  & Lower
  & $\log\!\left(\tfrac{1}{K}\sum_{k=1}^{K}p_{\mathrm{model}}(x| \pi^{(k)})\right)$
  & \cellcolor{goodbg}\textcolor{goodfg}{\cmark} \\
\midrule
CUBO$_{\beta\geq1}$ 
  & Upper
  & $\tfrac{1}{\beta}\,\badterm{\log}\!\left(\tfrac{1}{K}\sum_{k=1}^{K}p_{\mathrm{model}}(x| \pi^{(k)})^{\beta}\right)$
  & \cellcolor{badbg}\textcolor{badfg}{\xmark} \\
TVO$^{U}_{\Lambda}$
  & Upper
  & $\tfrac{1}{\Lambda}\sum_{\lambda=1}^{\Lambda}\sum_{k=1}^{K}\badterm{\bar{w}_{k}^{(\beta_{\lambda})}}\,\log p_{\mathrm{model}}(x| \pi^{(k)})$
  & \cellcolor{badbg}\textcolor{badfg}{\xmark} \\
IS-VG-B
  & Upper
  & $\tfrac{1}{n_p}\sum_{j=1}^{n_p}\badterm{\log}\!\left(\tfrac{1}{s}\sum_{i=1}^{s} p_{\mathrm{model}}(x|\pi^{(j,i)})\right) + \badterm{\log}\!\left(\tfrac{1}{n_p}\sum_{j=1}^{n_p}\dfrac{\sum_{i=1}^{s} p_{\mathrm{model}}(x|\widetilde\pi^{(j,i)})}{\sum_{i=1}^{s} p_{\mathrm{model}}(x|\pi^{(j,i)})}\right)$
  & \cellcolor{badbg}\textcolor{badfg}{\xmark} \\
\midrule
\cellcolor{goodbg}\textbf{TUBE (Ours)}
  & \cellcolor{goodbg} Upper
  & \cellcolor{goodbg} $\log\psi(x)+\goodterm{\dfrac{\tfrac{1}{K}\sum_{k=1}^{K} p_{\mathrm{model}}(x| \pi^{(k)}) - \psi(x)}{\psi(x)}}$
  & \cellcolor{goodbg}\textcolor{goodfg}{\cmark} \\
\bottomrule
\end{tabular}%
}
\end{table}

\section{Experiments}
\label{sec:experiments}
We demonstrate our TUBE through two categories of experiments. First, we show that TUBE is tight enough and helps to compare AO-ARM/MDM with ARMs (\S\ref{sec:owt_likelihood_comparison}). Second, we ablate the tightness of TUBE with respect to the choice of surrogate $\psi$ (\S\ref{sec:tightness}). The code is available in the supplementary.

All evaluations are done with BMs \citep{arriola2025block} at block sizes $L'\in{4,8,16}$ on OWT and LM1B \citep{Gokaslan2019OpenWeb,chelba2014billionwordbenchmarkmeasuring}. Following standard practice for MDMs (\S\ref{sec:mdm}), we report results in terms of perplexity $\mathrm{PPL}=\exp(-\log p_{\rm model})$ rather than log-likelihood directly. \textit{In this form, upper bounds for log-likelihood result in lower bounds for PPL, and vice-versa, lower log-likelihood bounds produce upper bounds for PPL.} For OWT, we use the available pretrained checkpoints. For LM1B, we train the BM from scratch using the original code, since checkpoints are unavailable. The conclusions are similar on both datasets, so we report only OWT results in the main text and defer \underline{LM1B results} to Appendix~\ref{app:replication_lm1b}. \underline{The experimental details} provided in Appendix~\ref{app:reproduction}.

\subsection{Likelihood comparison}
\label{sec:owt_likelihood_comparison}

\begin{table*}[t]
\centering
\small
\caption{\textbf{Test perplexity (\(\downarrow\)) on OWT.} We report $\mathrm{PPL}=\exp(-\log p_{\mathrm{model}}(x))$. The $\pm$ values denote standard deviations over $10$ random subsets of orderings, with all estimators in a row computed using the same total number of sampled orderings. The reference column reports $\mathrm{ELBO}_K$, which is exact for $L'=4$ (marked with $^{\star}$) and Monte Carlo-estimated for $L'\in\{8,16\}$. \textbf{Color coding:} coral background marks biased estimators, and red text marks cells where the value exceeds $\mathrm{ELBO}_K$, violating the expected bound relation.}
\label{tab:owt_ppl_main}
\renewcommand{\arraystretch}{1.15}
\setlength{\tabcolsep}{4.5pt}
\resizebox{\textwidth}{!}{%
\begin{tabular}{
@{} c l
>{\columncolor{plugcoral}}c
>{\columncolor{plugcoral}}c
>{\columncolor{plugcoral}}c
>{\columncolor{tubegreen}}c
>{\columncolor{refgray}}c
>{\columncolor{elbolavender}}c
>{\columncolor{gapamber}}c
@{}}
\toprule
& &
\multicolumn{3}{c}{\cellcolor{plugcoraldeep}\textbf{Biased upper estimators for log-likelihood}} &
\multicolumn{1}{c}{\cellcolor{tubedeep}\textcolor{white}{\textbf{Ours}}} &
\multicolumn{2}{c}{\textbf{References}} &
\multicolumn{1}{c}{\textbf{Gap}} \\
\cmidrule(lr){3-5}
\cmidrule(lr){6-6}
\cmidrule(lr){7-8}
\cmidrule(lr){9-9}
\textbf{$L'$} & \textbf{Regime}
& \textbf{CUBO$_{\beta=2}$}
& \textbf{TVO$_U$}
& \textbf{IS-VG-B}
& \textbf{TUBE}
& \textbf{ELBO$_K$}
& \textbf{ELBO}
& \textbf{|TUBE$-$ARM|} \\
\midrule
\rowcolor{baselinecream}
\multicolumn{9}{l}{\textit{Exact ARM baseline:} \(\textbf{17.54}\) PPL} \\
\midrule
\multirow{4}{*}{4}
& NFE$=1$
    & \(97.78\) & \(97.78\) & \(97.78\) & \(97.78\) & \(97.78\) & \(-\) & \textcolor{darkamber}{\(80.24\)} \\
& NFE$=2$
    & \(24.34 \pm 0.01\) & \textcolor{violationred}{\(27.68 \pm 0.01\)} & \textcolor{violationred}{\(29.11 \pm 0.03\)} & \(\geq 21.55 \pm 0.03\) & \(\leq 27.40\) & \(-\) & \textcolor{darkamber}{\(4.01\)} \\
& NFE$=4$
    & \(19.78 \pm 0.01\) & \textcolor{violationred}{\(21.89 \pm 0.01\)} & \textcolor{violationred}{\(22.52 \pm 0.01\)} & \(\geq 19.16 \pm 0.02\) & \(\leq 21.63\) & \(-\) & \textcolor{darkamber}{\(1.62\)} \\
& AO-ARM
    & \(17.50 \pm 0.00\) & \textcolor{violationred}{\(18.55 \pm 0.00\)} & \textcolor{violationred}{\(18.76 \pm 0.01\)} & \(\geq 17.74 \pm 0.01\) & \(=18.46^{\star}\) & \(\leq 21.36 \pm 0.21\) & \textcolor{darkamber}{\(0.20\)} \\
\midrule
\multirow{5}{*}{8}
& NFE$=1$
    & \(221.54\) & \(221.54\) & \(221.54\) & \(221.54\) & \(221.54\) & \(-\) & \textcolor{darkamber}{\(204.00\)} \\
& NFE$=2$
    & \(29.04 \pm 0.01\) & \textcolor{violationred}{\(33.41 \pm 0.02\)} & \textcolor{violationred}{\(34.95 \pm 0.03\)} & \(\geq 24.32 \pm 0.06\) & \(\leq 32.63\) & \(-\) & \textcolor{darkamber}{\(6.78\)} \\
& NFE$=4$
    & \(20.90 \pm 0.01\) & \textcolor{violationred}{\(23.54 \pm 0.01\)} & \textcolor{violationred}{\(24.15 \pm 0.01\)} & \(\geq 20.23 \pm 0.03\) & \(\leq 23.25\) & \(-\) & \textcolor{darkamber}{\(2.69\)} \\
& NFE$=8$
    & \(18.85 \pm 0.01\) & \textcolor{violationred}{\(20.83 \pm 0.01\)} & \textcolor{violationred}{\(21.17 \pm 0.01\)} & \(\geq 19.15 \pm 0.02\) & \(\leq 20.67\) & \(-\) & \textcolor{darkamber}{\(1.61\)} \\
& AO-ARM
    & \(17.95 \pm 0.00\) & \textcolor{violationred}{\(19.31 \pm 0.00\)} & \textcolor{violationred}{\(19.46 \pm 0.00\)} & \(\geq 18.67 \pm 0.01\) & \(\leq 19.24\) & \(\leq 22.18 \pm 0.58\) & \textcolor{darkamber}{\(1.13\)} \\
\midrule
\multirow{6}{*}{16}
& NFE$=1$
    & \(398.63\) & \(398.63\) & \(398.63\) & \(398.63\) & \(398.63\) & \(-\) & \textcolor{darkamber}{\(381.09\)} \\
& NFE$=2$
    & \(35.80 \pm 0.01\) & \textcolor{violationred}{\(39.80 \pm 0.01\)} & \textcolor{violationred}{\(41.93 \pm 0.06\)} & \(\geq 22.50 \pm 0.15\) & \(\leq 38.43\) & \(-\) & \textcolor{darkamber}{\(4.96\)} \\
& NFE$=4$
    & \(22.89 \pm 0.02\) & \textcolor{violationred}{\(25.21 \pm 0.01\)} & \textcolor{violationred}{\(25.96 \pm 0.02\)} & \(\geq 19.03 \pm 0.08\) & \(\leq 24.75\) & \(-\) & \textcolor{darkamber}{\(1.49\)} \\
& NFE$=8$
    & \(19.59 \pm 0.01\) & \textcolor{violationred}{\(21.38 \pm 0.01\)} & \textcolor{violationred}{\(21.77 \pm 0.01\)} & \(\geq 18.44 \pm 0.04\) & \(\leq 21.14\) & \(-\) & \textcolor{darkamber}{\(0.90\)} \\
& NFE$=16$
    & \(18.44 \pm 0.01\) & \textcolor{violationred}{\(19.98 \pm 0.00\)} & \textcolor{violationred}{\(20.26 \pm 0.01\)} & \(\geq 18.22 \pm 0.03\) & \(\leq 19.83\) & \(-\) & \textcolor{darkamber}{\(0.68\)} \\
& AO-ARM
    & \(18.08 \pm 0.00\) & \textcolor{violationred}{\(19.39 \pm 0.00\)} & \textcolor{violationred}{\(19.56 \pm 0.01\)} & \(\geq 18.47 \pm 0.01\) & \(\leq 19.30\) & \(\leq 22.81 \pm 0.61\) & \textcolor{darkamber}{\(0.93\)} \\
\bottomrule
\end{tabular}%
}
\end{table*}

Here we ask two questions: is the true log-likelihood of AO-ARM/MDM competitive with an ARM baseline, and are existing lower-side estimators on log-likelihood tight enough to settle the comparison? We address both by pairing ELBO$_K$ with TUBE and benchmarking against prior estimators on a shared set of Monte Carlo samples.

\paragraph{Setup.} We compare TUBE against $\mathrm{CUBO}_{\beta=2}$, $\mathrm{TVO}_U$, IS-VG-B, and also report ELBO and $\mathrm{ELBO}_K$. The same pretrained BM is evaluated under two regimes: AO-ARM ($\pi\in\mathcal{S}$) and MDM ($\pi\in\mathcal{G}$, $\mathrm{NFE}\in\{1,2,4,\dots\}$). For $L'=4$ in AO-ARM we enumerate all $4!=24$ orderings, which gives the \textbf{exact} per-block log-likelihood and we report it instead of $\mathrm{ELBO}_K$. For $L'\in\{8,16\}$ this is intractable, so we sample $64$ and $128$ orderings respectively. Then, split equally into two independent sets used for $\widehat{p}_{\mathrm{model},K}$ and $\psi_M$, with $\mathrm{ELBO}_K$ reported as a Monte Carlo estimate.

\paragraph{Results.} For MDMs, Table~\ref{tab:owt_ppl_main} shows that TUBE is generally the tightest lower bound on PPL across all NFEs ($=T$ in \S\ref{sec:mdm}), compared with the competing estimators, with $\mathrm{CUBO}_{\beta=2}$ being the only exception. However, we emphasize that $\mathrm{CUBO}_{\beta=2}$ is  not reliable, as discussed in the next paragraph. When comparing MDMs with the ARM baseline, the gap in the rightmost column, decreases from about $5$ to $1$ PPL as NFE increases. For AO-ARMs, TUBE is the tightest estimator across the reported settings. The conclusions are the clearest for $L'=4$, where $\mathrm{ELBO}_K$ is exact, so the likelihood gap between AO-ARMs/MDMs and ARMs is localized without Monte Carlo ambiguity.

\begin{tcolorbox}[
  colback=white,
  colframe=gray!60!black,
  boxrule=0.6pt,
  arc=2pt,
  left=6pt,
  right=6pt,
  top=4pt,
  bottom=4pt
]
\textbf{Takeaway.} In every configuration the AO-ARM/MDM PPL lies strictly above the ARM, so ARM is dominant in likelihood across all block sizes and generation regimes we consider.
\end{tcolorbox}



The other estimators behave differently. $\mathrm{TVO}_U$ and IS-VG-B sit \emph{above} $\mathrm{ELBO}_K$ in every regime with $\mathrm{NFE}\geq 2$, \textbf{so they fail to give a reliable bound}. This is structural: each applies a nonlinear function after Monte Carlo averaging, so the population bound does not transfer to its empirical estimator. TUBE avoids this by being affine in $\widehat{p}_{\mathrm{model},K}$, so the bound from \eqref{eq:tube} transfers to the Monte Carlo estimator at every $K\geq 1$. $\mathrm{CUBO}_{\beta=2}$ behaves differently from $\mathrm{TVO}_U$ and IS-VG-B: in our data it sits below $\mathrm{ELBO}_K$ in every row. Nevertheless, this does not make $\mathrm{CUBO}_{\beta=2}$ a valid upper bound: as discussed in \S\ref{sec:related}, its empirical estimator applies $\log$ to a Monte Carlo mean and is therefore biased downward. Furthermore, the choice of $\beta$ can substantially \underline{change the estimator bias} and even lead to values exceeding $\mathrm{ELBO}_K$, see Appendix~\ref{app:cubo}. This makes CUBO \textbf{unreliable} as well.

\subsection{Surrogate $\psi$ choice and finite-sample behavior}
\label{sec:tightness}
By \eqref{eq:tube}, the tightness of TUBE depends on two factors: the number of orderings used to estimate $\widehat{p}_{\mathrm{model},K}$, and the surrogate $\psi$ used to construct the bound. While the Monte Carlo error in estimating $p_{\mathrm{model}}$ decreases as $1/\sqrt{K}$, the role of $\psi$ is less straightforward. So we ask which construction from \S\ref{sec:choice_surrogate_method} closes gap between $\mathrm{TUBE}_\psi$ and $\log p_{\mathrm{model}}(x)$ most effectively. 

\begin{figure}[t]
    \centering
    \noindent\makebox[\textwidth][c]{%
        \hspace*{-0.33cm}%
        \includegraphics[width=1.06\textwidth]{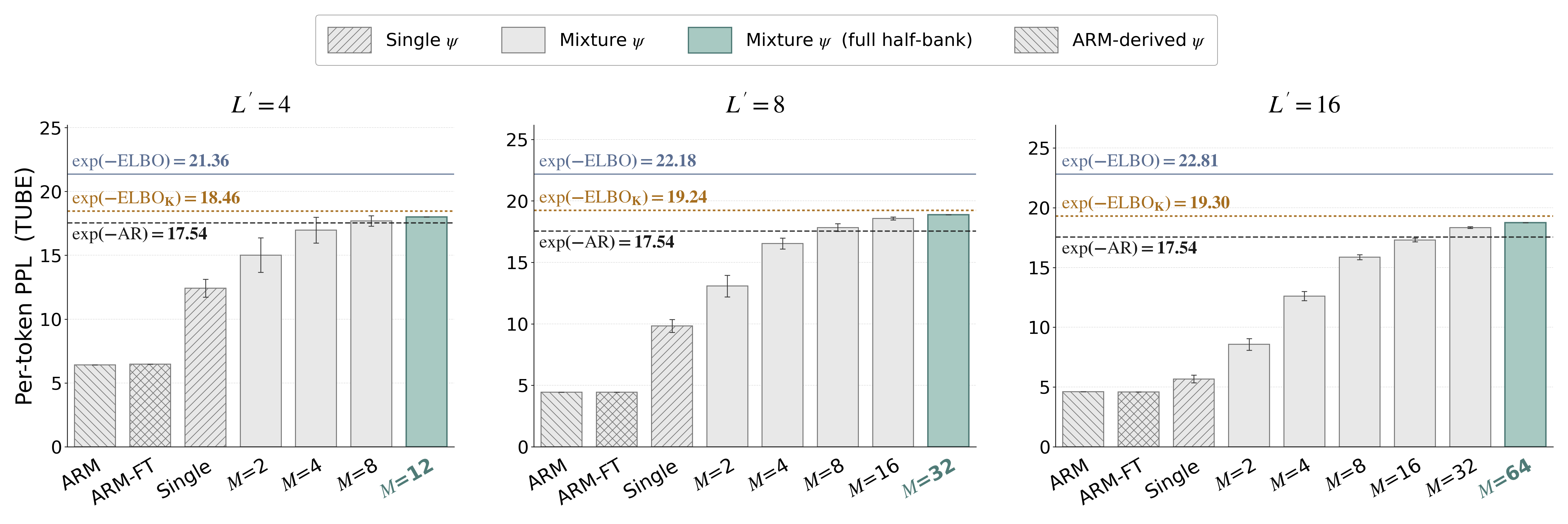}%
    }
    \caption{\textbf{Test perplexity ($\downarrow$) with different $\psi$ choices on OWT.} We report $\mathrm{PPL}=\exp(-\log p_{\mathrm{model}}(x))$. The bar plot reports standard deviations over 10 repeated ordering samples with different random seeds. The dotted gold reference line shows $\mathrm{ELBO}_K$.}
    \label{fig:psi_choice}
\end{figure}


\paragraph{Setup.}
With $\widehat{p}_{\mathrm{model},K}$ fixed, we compare four self-surrogates: a single-order surrogate $\psi_{\pi}$, the self-surrogate $\psi_M$, an external ARM surrogate $\psi_{\mathrm{ARM}}$, and a variant of ARM fine-tuned on samples generated by the BM (Appendix~\ref{app:reproduction}). We vary the number of sampled orderings $M$ on a $\log_2$ grid up to the full ordering budget used in \S\ref{sec:owt_likelihood_comparison}.

\paragraph{Results.} Figure~\ref{fig:psi_choice} shows that the self-surrogate $\psi_M$ gets closest to $\mathrm{ELBO}_K$ at every block size with increase of $M$. The external ARM surrogate $\psi_{\mathrm{ARM}}$ leaves a much wider gap, and fine-tuning it on BM samples shifts the bound only marginally. A single fixed order $\psi_\pi$ is looser still. This matches the discussion of \S\ref{sec:choice_surrogate_method}: $\psi_M$ averages $p_{\mathrm{model}}(x|\pi)$ over independent orderings drawn from the same model and therefore tracks $p_{\mathrm{model}}(x)$ pointwise better than an external density.

\section{Potential impact and limitations}
\label{sec:impact}

\paragraph{Limitations.} Our empirical study is limited to BM~\citep{arriola2025block} which is the dominant inference-time paradigm of large-scale MDMs such as LLaDA~2.0~\citep{bie2025llada2} and Mercury~\citep{khanna2025mercury}. However, applying TUBE to non-block MDM may be challenging. TUBE's tightness requires the surrogate $\psi(x) \approx p_{\mathrm{model}}(x)$. For non-block MDM at typical sequence lengths $L \sim 10^2$ to $10^3$, $p_{\mathrm{model}}(x)$ is approximately product of $L$ per-token probabilities, vanishingly small in absolute scale. Thus, $\psi$  becomes hard to estimate reliably in practice. In turn, BM decomposition \eqref{eq:block_decomp} for moderate block size keeps computing log-likelihood numerically feasible.

\paragraph{Potential impact.}
The reliable two-sided estimation of $\log p_{\mathrm{model}}(x)$ enables several downstream uses. It tightens perplexity bounds for diffusion language models~\citep{sahoo2024mdlm,arriola2025block}, supplies finite-sample policy ratios for RL post-training~\citep{christiano2017rlhf,ouyang2022instructgpt,shao2024deepseekmath,rafailov2023dpo}, and supports likelihood-based evaluation in scientific generative modeling such as protein sequence design~\citep{alamdari2023evodiff,wang2025dplm2}. Studying these directions using our TUBE is a promising avenue for the future work.


\bibliographystyle{plainnat}
\bibliography{references.bib}


\appendix








\newpage 

\section{Experiments on LM1B}
\label{app:replication_lm1b}
This section reports the results of our TUBE evaluation on LM1B dataset (Table \ref{tab:lm1b_ppl_main} and Figure \ref{fig:group_size_lm1b}). The conclusions are analogous to those from the main paper obtained on OWT dataset.

\begin{table*}[!h]
\centering
\small
\caption{\textbf{Test perplexity (\(\downarrow\)) on LM1B.} We report $\mathrm{PPL}=\exp(-\log p_{\mathrm{model}}(x))$. The $\pm$ values denote standard deviations over $10$ random subsets of orderings, with all estimators in a row computed using the same total number of sampled orderings. The reference column reports $\mathrm{ELBO}_K$, which is exact for $L'=4$ (marked with $^{\star}$) and Monte Carlo-estimated for $L'\in\{8,16\}$. \textbf{Color coding:} coral background marks biased estimators, and red text marks cells where the value exceeds $\mathrm{ELBO}_K$, violating the expected bound relation.}
\label{tab:lm1b_ppl_main}
\renewcommand{\arraystretch}{1.15}
\setlength{\tabcolsep}{4.5pt}
\resizebox{\textwidth}{!}{%
\begin{tabular}{
@{} c l
>{\columncolor{plugcoral}}c
>{\columncolor{plugcoral}}c
>{\columncolor{plugcoral}}c
>{\columncolor{tubegreen}}c
>{\columncolor{refgray}}c
>{\columncolor{elbolavender}}c
>{\columncolor{gapamber}}c
@{}}
\toprule
& &
\multicolumn{3}{c}{\cellcolor{plugcoraldeep}\textbf{Biased upper estimators}} &
\multicolumn{1}{c}{\cellcolor{tubedeep}\textcolor{white}{\textbf{Ours}}} &
\multicolumn{2}{c}{\textbf{References}} &
\multicolumn{1}{c}{\textbf{Gap}} \\
\cmidrule(lr){3-5}
\cmidrule(lr){6-6}
\cmidrule(lr){7-8}
\cmidrule(lr){9-9}
\textbf{$L'$} & \textbf{Regime}
& \textbf{CUBO$_{\beta=2}$}
& \textbf{TVO$_U$}
& \textbf{IS-VG-B}
& \textbf{TUBE}
& \textbf{ELBO$_K$}
& \textbf{ELBO}
& \textbf{TUBE$-$ARM} \\
\midrule
\rowcolor{baselinecream}
\multicolumn{9}{l}{\textit{Exact ARM baseline:} \(22.46\) PPL} \\
\midrule
\multirow{4}{*}{4}
& NFE$=1$
    & \(137.09\) & \(137.09\) & \(137.09\) & \(137.09\) & \(137.09\) & \(-\) & \textcolor{darkamber}{\(114.62\)} \\
& NFE$=2$
    & \(34.67 \pm 0.01\) & \textcolor{violationred}{\(39.45 \pm 0.01\)} & \textcolor{violationred}{\(41.40 \pm 0.04\)} & \(\geq 30.87 \pm 0.03\) & \(\leq 39.08\) & \(-\) & \textcolor{darkamber}{\(8.41\)} \\
& NFE$=4$
    & \(28.05 \pm 0.01\) & \textcolor{violationred}{\(31.01 \pm 0.01\)} & \textcolor{violationred}{\(31.85 \pm 0.02\)} & \(\geq 27.34 \pm 0.03\) & \(\leq 30.67\) & \(-\) & \textcolor{darkamber}{\(4.88\)} \\
& AO-ARM
    & \(25.25 \pm 0.00\) & \textcolor{violationred}{\(26.54 \pm 0.00\)} & \textcolor{violationred}{\(26.78 \pm 0.00\)} & \(\geq 25.63 \pm 0.01\) & \(=26.45^{\star}\) & \(\leq 28.95 \pm 0.38\) & \textcolor{darkamber}{\(3.17\)} \\
\midrule
\multirow{5}{*}{8}
& NFE$=1$
    & \(299.88\) & \(299.88\) & \(299.88\) & \(299.88\) & \(299.88\) & \(-\) & \textcolor{darkamber}{\(277.42\)} \\
& NFE$=2$
    & \(41.45 \pm 0.02\) & \textcolor{violationred}{\(47.68 \pm 0.02\)} & \textcolor{violationred}{\(49.73 \pm 0.03\)} & \(\geq 35.24 \pm 0.05\) & \(\leq 46.66\) & \(-\) & \textcolor{darkamber}{\(12.77\)} \\
& NFE$=4$
    & \(29.78 \pm 0.01\) & \textcolor{violationred}{\(33.53 \pm 0.01\)} & \textcolor{violationred}{\(34.34 \pm 0.02\)} & \(\geq 29.11 \pm 0.04\) & \(\leq 33.14\) & \(-\) & \textcolor{darkamber}{\(6.65\)} \\
& NFE$=8$
    & \(26.88 \pm 0.01\) & \textcolor{violationred}{\(29.65 \pm 0.01\)} & \textcolor{violationred}{\(30.09 \pm 0.01\)} & \(\geq 27.51 \pm 0.02\) & \(\leq 29.44\) & \(-\) & \textcolor{darkamber}{\(5.05\)} \\
& AO-ARM
    & \(25.14 \pm 0.00\) & \textcolor{violationred}{\(26.90 \pm 0.00\)} & \textcolor{violationred}{\(27.09 \pm 0.01\)} & \(\geq 26.15 \pm 0.01\) & \(\leq 26.82\) & \(\leq 30.32 \pm 0.38\) & \textcolor{darkamber}{\(3.69\)} \\
\midrule
\multirow{6}{*}{16}
& NFE$=1$
    & \(546.32\) & \(546.32\) & \(546.32\) & \(546.32\) & \(546.32\) & \(-\) & \textcolor{darkamber}{\(523.86\)} \\
& NFE$=2$
    & \(51.26 \pm 0.03\) & \textcolor{violationred}{\(57.02 \pm 0.03\)} & \textcolor{violationred}{\(59.97 \pm 0.03\)} & \(\geq 31.99 \pm 0.13\) & \(\leq 55.13\) & \(-\) & \textcolor{darkamber}{\(9.53\)} \\
& NFE$=4$
    & \(32.81 \pm 0.02\) & \textcolor{violationred}{\(36.15 \pm 0.02\)} & \textcolor{violationred}{\(37.18 \pm 0.02\)} & \(\geq 27.39 \pm 0.06\) & \(\leq 35.52\) & \(-\) & \textcolor{darkamber}{\(4.93\)} \\
& NFE$=8$
    & \(28.00 \pm 0.01\) & \textcolor{violationred}{\(30.54 \pm 0.01\)} & \textcolor{violationred}{\(31.07 \pm 0.01\)} & \(\geq 26.60 \pm 0.04\) & \(\leq 30.23\) & \(-\) & \textcolor{darkamber}{\(4.14\)} \\
& NFE$=16$
    & \(26.34 \pm 0.01\) & \textcolor{violationred}{\(28.50 \pm 0.01\)} & \textcolor{violationred}{\(28.86 \pm 0.01\)} & \(\geq 26.24 \pm 0.02\) & \(\leq 28.30\) & \(-\) & \textcolor{darkamber}{\(3.78\)} \\
& AO-ARM
    & \(25.14 \pm 0.00\) & \textcolor{violationred}{\(26.88 \pm 0.00\)} & \textcolor{violationred}{\(27.09 \pm 0.01\)} & \(\geq 25.82 \pm 0.02\) & \(\leq 26.78\) & \(\leq 31.03 \pm 0.49\) & \textcolor{darkamber}{\(3.36\)} \\
\bottomrule
\end{tabular}%
}
\end{table*}

\begin{figure}[!h]
    \centering
    \noindent\makebox[\textwidth][c]{%
        \hspace*{-0.3cm}%
        \includegraphics[width=1.06\textwidth]{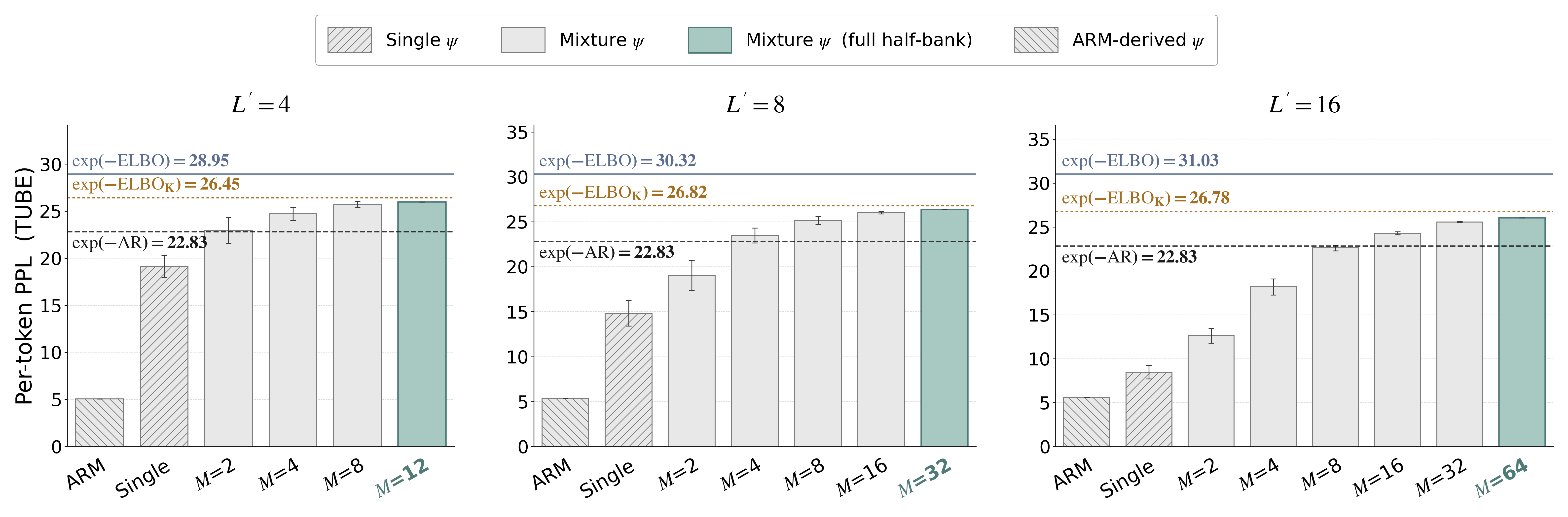}%
    }
    \caption{\textbf{Test perplexity ($\downarrow$) with different $\psi$ choices on LM1B.} We report $\mathrm{PPL}=\exp(-\log p_{\mathrm{model}}(x))$. The bar plot reports standard deviations over 10 repeated ordering samples with different random seeds. The dotted gold reference line shows $\mathrm{ELBO}_K$.}
    \label{fig:group_size_lm1b}
\end{figure}

\section{Additional methodology}

\subsection{TUBE Estimation Algorithm}\label{app:add_method}

Here we present the Algorithm~\ref{alg:tube} for the two-sided likelihood localization via our TUBE$_\psi$ \eqref{eq:tube} and ELBO$_K$ estimate \cite{sahoo2025esoteric} which originates from I-WAE model \cite{burda2016iwae}.


\begin{algorithm}[h]
\caption{TUBE$_\psi$ and ELBO$_K$: two-sided likelihood interval}
\label{alg:tube}
\begin{algorithmic}[1]
\Require sample $x$, surrogate type $c\in\{\mathrm{ARM},\mathrm{self}\}$, sample budgets $K,M\ge 1$, model $p_{\rm model}(\cdot|\pi)$, latent prior $p(\pi)$, optional ARM surrogate $p_{\rm ARM}$
\State Sample $\pi^{(1)},\dots,\pi^{(K)}\overset{\mathrm{i.i.d.}}{\sim}p(\pi)$
\State $\widehat{p}_{\mathrm{model},K}(x)\gets \dfrac{1}{K}\sum_{k=1}^{K} p_{\rm model}(x|\pi^{(k)})$
\State $\widehat{\mathrm{TUBE}}_{\psi,K}(x)\gets \log\psi(x)+\dfrac{\widehat{p}_{\mathrm{model},K}(x)-\psi(x)}{\psi(x)}$
\State $\widehat{\mathrm{ELBO}}_{K}(x)\gets \log \widehat{p}_{\mathrm{model},K}(x)$
\State \Return LL interval $[\,\widehat{\mathrm{ELBO}}_{K}(x),\,\widehat{\mathrm{TUBE}}_{\psi,K}(x)\,]$
\end{algorithmic}
\end{algorithm}

\subsection{Statistical properties}

\begin{proposition}[Unbiasednes and Variance of TUBE]
\label{prop:tube_variance}
Fix $x \in \mathcal X$ and assume that $\psi(x)>0$ is deterministic. Recap the  Monte Carlo estimator of $\mathrm{TUBE}_{\psi}(x)$:
\begin{equation}
    \widehat{\mathrm{TUBE}_{\psi,K}}(x)\!:=\!\log\psi(x)+\frac{\widehat{p}_{\mathrm{model},K}(x)-\psi(x)}{\psi(x)},\quad \widehat{p}_{\mathrm{model},K}(x):=\frac{1}{K}\sum_{k=1}^{K} p_{\rm model}(x| \pi^{(k)}),
\end{equation}
Then, for every $K \ge 1$ and $\pi^{(1)},\ldots,\pi^{(K)}$ which are i.i.d.:
\[
    \mathbb E_{\pi^{(1:K)}}\!
    \left[
        \widehat{\mathrm{TUBE}}_{\psi,K}(x)
    \right]
    =
    \mathrm{TUBE}_{\psi}(x),
\]
Moreover, if
$\operatorname{Var}_{\pi \sim p(\pi)}[p_{\mathrm{model}}(x \mid \pi)] < \infty$, then
\[
    \operatorname{Var}_{\pi^{(1:K)}}\!
    \left[
        \widehat{\mathrm{TUBE}}_{\psi,K}(x)
    \right]
    =
    \frac{1}{K\,\psi(x)^2}
    \operatorname{Var}_{\pi \sim p(\pi)}
    \left[
        p_{\mathrm{model}}(x \mid \pi)
    \right].
\]
\end{proposition}

\begin{proof}
Since $\pi^{(1)},\ldots,\pi^{(K)}$ are i.i.d.,
\[
    \mathbb E_{\pi^{(1:K)}}[
        \widehat p_{\mathrm{model},K}(x)
    ]
    =
    \mathbb E_{\pi \sim p(\pi)}
    [
        p_{\mathrm{model}}(x\mid \pi)
    ]
    =
    p_{\mathrm{model}}(x).
\]
Therefore,
\[
    \mathbb E_{\pi^{(1:K)}}\!
    \left[
        \widehat{\mathrm{TUBE}}_{\psi,K}(x)
    \right]
    =
    \log \psi(x)
    +
    \frac{
        \mathbb E[
            \widehat p_{\mathrm{model},K}(x)
        ]
        -
        \psi(x)
    }{\psi(x)}
    =
    \log \psi(x)
    +
    \frac{
        p_{\mathrm{model}}(x)
        -
        \psi(x)
    }{\psi(x)}
    =
    \mathrm{TUBE}_{\psi}(x).
\]

For the variance, the only random term in
$\widehat{\mathrm{TUBE}}_{\psi,K}(x)$ is
$\widehat p_{\mathrm{model},K}(x)$, since $\psi(x)$ is fixed. Hence:
\[
    \operatorname{Var}
    \left[
        \widehat{\mathrm{TUBE}}_{\psi,K}(x)
    \right]
    =
    \frac{1}{\psi(x)^2}
    \operatorname{Var}
    \left[
        \widehat p_{\mathrm{model},K}(x)
    \right].
\]
By independence of the $K$ samples,
\[
    \operatorname{Var}
    \left[
        \widehat p_{\mathrm{model},K}(x)
    \right]
    =
    \operatorname{Var}
    \left[
        \frac{1}{K}\sum_{k=1}^{K} p_{\rm model}(x| \pi^{(k)})
    \right]
    =
    \frac{1}{K}
    \operatorname{Var}_{\pi \sim p(\pi)}
    [
        p_{\mathrm{model}}(x\mid \pi)
    ].
\]
\end{proof}

\section{Comparison of likelihood estimators}
\label{app:estimator_comparison}
In this appendix we give the precise definitions of the upper-bound estimators outlined in \S\ref{sec:related} (CUBO$_\beta$, TVO$_\Lambda$, IS-VG-B), derive their finite-sample Monte Carlo forms, and identify the structural source of bias in each. Throughout, we use the notation of \S\ref{sec:method}: $\pi$ is the latent generation structure (an order $\pi\in\mathcal{S}$ for AO-ARM, a grouped ordering $\pi\in\mathcal{G}$ for MDM), $p(\pi)$ is a distribution on orders defined by the model, and the standard $K$-sample Monte Carlo estimator $\widehat{p}_{\mathrm{model},K}(x)$ is the one defined in \eqref{eq:tube_mc}. Each estimator below uses this single shared sample bank.


\subsection{CUBO and the R\'enyi variational bound}
\label{app:cubo}

The $\chi^{\beta}$-divergence upper bound (CUBO) of \citep{dieng2017chivi} and the
R\'enyi variational bound of \citep{li2016renyi} are two equivalent formulations
of the same family of upper bounds on $\log p_{\text{model}}(x)$. In our latent-mixture form,
\begin{equation}
    \mathrm{CUBO}_\beta(x)
    \;=\;
    \frac{1}{\beta}\log
    \mathbb{E}_{\pi\sim p(\pi)}\!\left[
        p_{\mathrm{model}}(x|\pi)^{\beta}
    \right]
    \;\ge\;
    \log p_{\mathrm{model}}(x), \qquad \beta\ge 1,
    \label{eq:app_cubo_pop}
\end{equation}
with equality at $\beta=1$ and monotone non-decreasing in $\beta$ \citep{dieng2017chivi}. 
\paragraph{Empirical Monte Carlo estimator.} It can be practically estimated (Table~\ref{tab:estimator_zoo}) as
\begin{equation}
    \widehat{\mathrm{CUBO}}_{\beta,K}(x)
    \;=\;
    \frac{1}{\beta}\log
    \!\left(
        \frac{1}{K}\sum_{k=1}^{K}
        p_{\mathrm{model}}(x|\pi^{(k)})^{\beta}
    \right).
    \label{eq:app_cubo_mc}
\end{equation}

\paragraph{Source of bias.}
The estimator \eqref{eq:app_cubo_mc} is biased at finite $K$: by Jensen on the outer $\log$, $\mathbb{E}[\widehat{\mathrm{CUBO}}_{\beta,K}] \le \mathrm{CUBO}_\beta$. The population guarantee $\log p_{\mathrm{model}}(x) \le \mathrm{CUBO}_\beta$ therefore does not transfer to the empirical estimator: the chain $\log p_{\mathrm{model}}(x) \le \mathbb{E}[\widehat{\mathrm{CUBO}}_{\beta,K}]$ is no longer ensured, so the empirical expectation may fall below $\log p_{\mathrm{model}}(x)$ and void the upper bound. Figure~\ref{fig:cubo_heatmap} shows configurations of $(|\mathcal{B}|, \beta)$ where this occurs in practice. SPG~\citep{wang2026spg} observes this and notes that the linearization $\log x \le x-1$ would restore the bound at the cost of looseness, motivating their tighter-but-biased choice.

\begin{figure}[!h]
    \centering
    \noindent\makebox[\textwidth][c]{\includegraphics[width=1.0\textwidth]{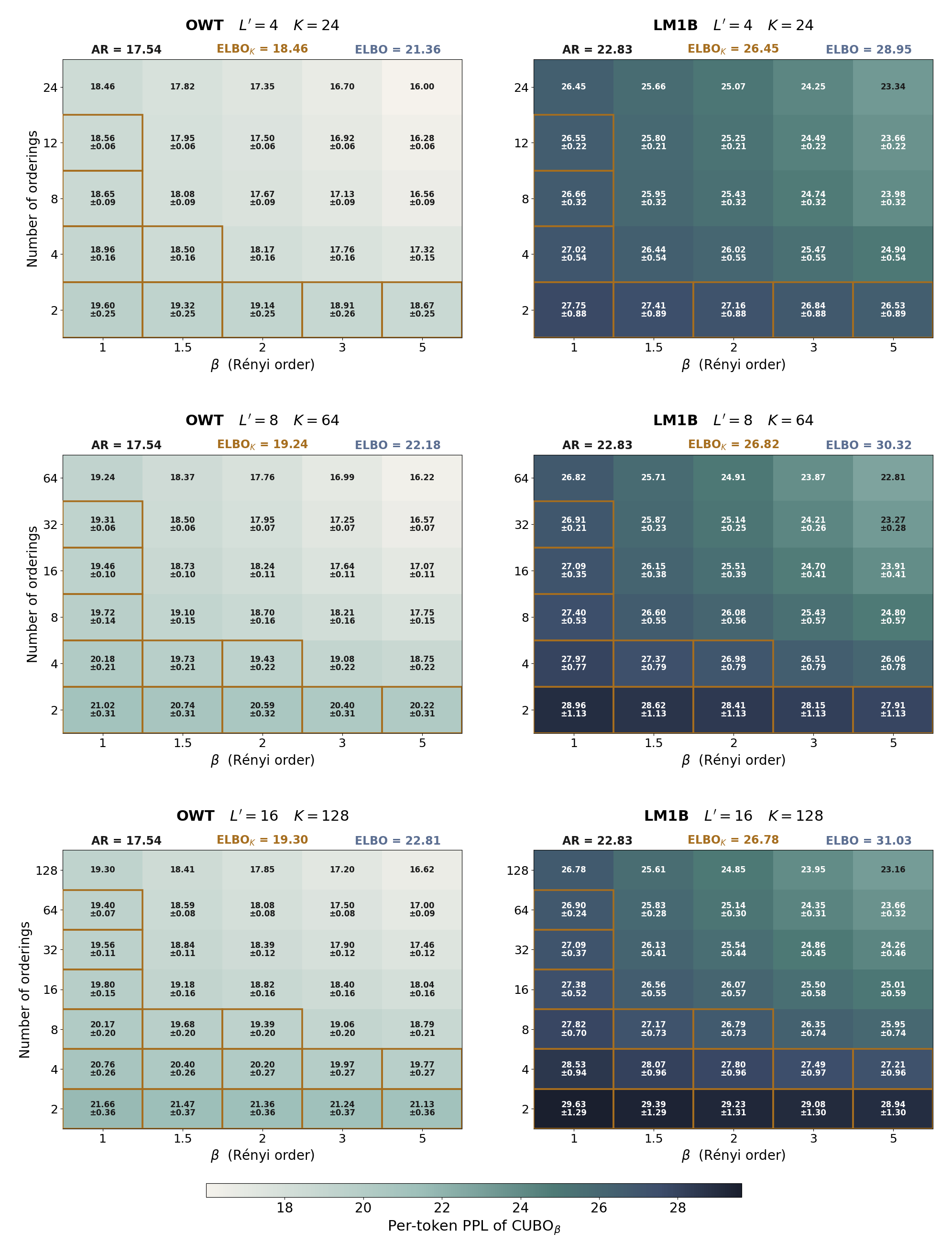}}
    \caption{\textbf{CUBO$_\beta$ diagnostic on $L'\in\{4,8,16\}$.} PPL of CUBO$_\beta$ as a function of the number of orderings and the Rényi exponent $\beta$, on OWT (left) and LM1B (right). Gold borders mark cells exceeding $\mathrm{ELBO}_K^{\mathrm{full}}$, where CUBO violates its nominal upper-bound guarantee. The only cell in the plane that is both valid and tight is $(\text{number of orderings = 24}, \beta=1)$, which coincides with $\mathrm{ELBO}_K^{\mathrm{full}}$ by construction; reducing number of orderings breaks the bound upward, increasing $\beta$ pushes it past the truth downward.}
    \label{fig:cubo_heatmap}
\end{figure}

\paragraph{Connection to SPG.}
SPG's ``evidence upper bound'' (EUBO) \citep{wang2026spg} is $\mathrm{CUBO}_\beta$ specialized to the MDLM absorbing forward process~\citep{sahoo2024mdlm}: SPG's Lemma~1 is \eqref{eq:app_cubo_pop} under $n\!\leftrightarrow\!\beta$, and Theorem~1 specializes it to the per-token MDLM kernel. This is distinct from the KL-based EUBO of \citep{ji2019eubo}, which requires posterior samples.

\subsection{TVO: Thermodynamic-integration Upper Bound}
\label{app:tvo}

The TVO of \citep{masrani2019tvo} starts from the thermodynamic identity
\begin{equation}
    \log p_{\mathrm{model}}(x)
    \;=\;
    \int_0^1
    \mathbb{E}_{\pi\sim q_\beta(\cdot|x)}
    \!\left[
        \log p_{\mathrm{model}}(x|\pi)
    \right]\,d\beta,
    \qquad
    q_\beta(\pi|x) \propto p_{\mathrm{model}}(x|\pi)^\beta,
    \label{eq:app_tvi}
\end{equation}
where $q_\beta$ interpolates between the prior $q_0 = p(\pi)$ and the posterior $q_1 = p_{\mathrm{model}}(\pi|x)$. The integrand is monotone in $\beta$~\citep{masrani2019tvo}, so the right Riemann sum on the equispaced grid $\beta_\lambda = \lambda/\Lambda$ upper-bounds the integral:
\begin{equation}
    \log p_{\mathrm{model}}(x)
    \;\le\;
    \frac{1}{\Lambda}\sum_{\lambda=1}^{\Lambda}
    \mathbb{E}_{\pi\sim q_{\beta_\lambda}(\cdot|x)}\!\left[\log p_{\mathrm{model}}(x|\pi)\right]
    \;=:\;
    \mathrm{TVO}^{U}_{\Lambda}(x).
    \label{eq:app_tvo_bound}
\end{equation}
\paragraph{Empirical Monte Carlo estimator.}
Using the same $K$ sampled orderings as in \eqref{eq:tube_mc} with self-normalized weights $\bar{w}_k^{(\beta)} = p_{\mathrm{model}}(x|\pi^{(k)})^\beta \,/\, \sum_{j} p_{\mathrm{model}}(x|\pi^{(j)})^\beta$, the empirical MC estimator (Table~\ref{tab:estimator_zoo}) is
\begin{equation}
    \widehat{\mathrm{TVO}}^{U}_{\Lambda,K}(x)
    \;=\;
    \frac{1}{\Lambda}\sum_{\lambda=1}^{\Lambda}\sum_{k=1}^{K}
    \bar{w}_k^{(\beta_\lambda)}\,\log p_{\mathrm{model}}(x|\pi^{(k)}).
    \label{eq:app_tvo_mc_upper}
\end{equation}

\paragraph{Sources of bias.}
Two structural sources. \emph{(i)} Self-normalized importance weights $\bar{w}_k^{(\beta)}$ have $\mathcal{O}(1/K)$ bias~\citep{owen2013monte}: as $\beta \to 1$ the target $q_\beta$ concentrates on a vanishing subset of orderings while the proposal stays uniform, so weights at large $\beta$ collapse onto a few outlier samples. As a result, $\mathbb{E}[\widehat{\mathrm{TVO}}^U_{\Lambda,K}] \neq \mathrm{TVO}^U_\Lambda$ at finite $K$, and the population upper-bound guarantee no longer ensures $\mathbb{E}[\widehat{\mathrm{TVO}}^U_{\Lambda,K}] \ge \log p_{\mathrm{model}}(x)$. \emph{(ii)} The Riemann discretization adds an $\mathcal{O}(1/\Lambda)$ residual that does not vanish at fixed $\Lambda$.


\subsection{IS-VG-B: Importance Sampling Variational Gap Bound}
\label{app:isvgb}

IS-VG-B \citep{struski2023bounding} constructs an upper bound on $\log p_{\mathrm{model}}(x)$ by adding a correction term to the (lower) $\mathrm{ELBO}_s$ bound. The strategy is to pair a two-point inequality with a tunable parameter $C$, optimize $C$ in closed form, and estimate the optimum using a second independent MC estimator drawn from the same distribution.

The starting two-point inequality is
\[
    \log\mathbb{E}X \;\le\; \mathbb{E}\log X - 1 + C + e^{-C}\,\mathbb{E}[Y/X],
\]
which holds for any $C \in \mathbb{R}$ which holds for any $C \in \mathbb{R}$ when $X$ and $Y$ are i.i.d. positive random variables, with optimum at $C^\star = \log\mathbb{E}[Y/X]$~\citep[Theorem~4 and Corollary~4]{struski2023bounding}. Substituting $C^\star$ collapses the right-hand side to $\mathbb{E}\log X + \log\mathbb{E}[Y/X]$. Replacing $X$ by the $s$-sample MC average $\bar{X}_s = \tfrac{1}{s}\sum_i X_i$ preserves the bound at each $s$ since $\mathbb{E}\bar{X}_s = \mathbb{E}X$, and $\log\mathbb{E}[\bar{Y}_s/\bar{X}_s] \to 0$ as $s \to \infty$ on bounded support~\citep[Theorem~5]{struski2023bounding}, so the bound becomes asymptotically tight. Specializing to our setting with $X \equiv p_{\mathrm{model}}(x|\pi)$ and uniform $p(\pi)$ gives
\begin{equation}
    \mathrm{IS\text{-}VG\text{-}B}_s(x)
    \;:=\;
    \underbrace{\mathbb{E}\!\left[\log\widehat{p}_{\mathrm{model},s}(x)\right]}_{=\;\mathrm{ELBO}_s}
    \;+\;
    \underbrace{\log\,\mathbb{E}\!\left[
        \frac{\widetilde{p}_{\mathrm{model},s}(x)}{\widehat{p}_{\mathrm{model},s}(x)}
    \right]}_{\text{correction term}}
    \;\ge\;
    \log p_{\mathrm{model}}(x),
    \label{eq:app_isvgb_pop}
\end{equation}
where $\widehat{p}_{\mathrm{model},s}$ and $\widetilde{p}_{\mathrm{model},s}$ are independent $s$-sample MC averages of $p_{\mathrm{model}}(x|\pi)$ as in \eqref{eq:tube_mc}. The first term is the $\mathrm{ELBO}_s$ \emph{lower} bound at sample size $s$, the correction is what flips the inequality into an \emph{upper} bound.

\paragraph{Empirical Monte Carlo estimator.}
The correction's outer expectation is itself estimated by Monte Carlo over $n_p$ independent pairs of $s$-sample MC estimators (total budget $K = 2\,s\,n_p$). Expanding the per-pair averages $\widehat{p}_{\mathrm{model},s}^{(j)}(x) = \tfrac{1}{s}\sum_{i=1}^{s} p_{\mathrm{model}}(x|\pi^{(j,i)})$ and $\widetilde{p}_{\mathrm{model},s}^{(j)}(x) = \tfrac{1}{s}\sum_{i=1}^{s} p_{\mathrm{model}}(x|\widetilde\pi^{(j,i)})$ explicitly recovers the form shown in Table~\ref{tab:estimator_zoo}:
\begin{equation}
    \widehat{\mathrm{IS\text{-}VG\text{-}B}}_{s,n_p}(x)
    \;=\;
    \frac{1}{n_p}\sum_{j=1}^{n_p}\log\!\left(\frac{1}{s}\sum_{i=1}^{s} p_{\mathrm{model}}(x|\pi^{(j,i)})\right) \;+\; \log\!\left(\frac{1}{n_p}\sum_{j=1}^{n_p}\frac{\sum_{i=1}^{s} p_{\mathrm{model}}(x|\widetilde\pi^{(j,i)})}{\sum_{i=1}^{s} p_{\mathrm{model}}(x|\pi^{(j,i)})}\right),
    \label{eq:app_isvgb_mc}
\end{equation}
where $\pi^{(j,i)}$ and $\widetilde\pi^{(j,i)}$ are the $i$-th i.i.d. samples of the $j$-th X-side and Y-side MC estimators, respectively, both drawn from $p(\pi)$. The first term averages $n_p$ $\mathrm{ELBO}_s$ samples, the second averages the ratios \emph{inside} the log to estimate the population correction.

\paragraph{Source of bias.}
The first term in \eqref{eq:app_isvgb_mc} is an MC estimate of $\mathrm{ELBO}_s$, which by Jensen lies below $\log p_{\mathrm{model}}(x)$ at any finite $s$, this gap is part of the construction, since the correction is designed precisely to lift the bound back above $\log p_{\mathrm{model}}(x)$. The second term is meant to do that lifting, but is itself a $\log$ of a Monte Carlo mean: letting $r^{(j)} = \sum_{i=1}^{s} p_{\mathrm{model}}(x|\widetilde\pi^{(j,i)}) \,\big/\, \sum_{i=1}^{s} p_{\mathrm{model}}(x|\pi^{(j,i)})$ denote the per-pair ratio, Jensen gives $\mathbb{E}\!\left[\log\tfrac{1}{n_p}\sum_j r^{(j)}\right] \le \log\mathbb{E}[r]$, so the empirical correction underestimates the population value and the empirical IS-VG-B can fall below $\log p_{\mathrm{model}}(x)$ at finite $K$. \citep{struski2023bounding} acknowledge this in their Limitations: the empirical estimator ``results in nonrigorous bounds.''

\section{Experimental details}
\label{app:reproduction}

This appendix provides the experimental details for the results reported in the main paper: models and datasets (\S\ref{app:models_and_checkpoints}), construction of the orders (\S\ref{app:permutations}), estimator hyperparameters used in the main tables (\S\ref{app:estimator_hyperparameters}), and compute resources (\S\ref{app:compute}).

\subsection{Models and datasets}
\label{app:models_and_checkpoints}

We evaluate the public BMs from~\citep{arriola2025block} (\url{https://github.com/kuleshov-group/bd3lms}). For OWT, we use publicly available checkpoints.


\textbf{LM1B training details.} For LM1B, we follow the original BD3-LM training procedure. We first train an MDLM with the BD3-LM repository's default hyperparameters (global batch size $512$, learning rate $3\!\times\!10^{-4}$, $150{,}000$ steps, sentence-wrap data preprocessing), and then fine-tune it separately for each block size $L'\in\{4,8,16\}$ to obtain the BD3-LM checkpoints. The ARM baseline is trained independently from scratch with the same architecture, tokenizer, and sequence length. The synthetic-finetuned ARM $\psi_{\mathrm{ARM\text{-}FT}}$ used in Figure~\ref{fig:psi_choice} is a one-off OWT experiment and is not available for LM1B, so the corresponding bar is missing from the LM1B panel of Figure~\ref{fig:group_size_lm1b}.

\subsection{Sampling orderings}
\label{app:permutations}

For each block size $L'$, we precompute a fixed set of $M$ latent orderings per block, with each ordering specifying the unmasking sequence over the $L'$ block positions. The bank is constructed once per $(L', \text{dataset})$ pair and shared across all estimators in the table. Number of latent orderings and construction protocol are summarized in Table~\ref{tab:reproduction_banks}.

\begin{table}[h]
\centering
\small
\caption{\textbf{Latent orderings construction per block size.}}
\label{tab:reproduction_banks}
\renewcommand{\arraystretch}{1.15}
\setlength{\tabcolsep}{6pt}
\begin{tabular}{@{} c l l @{}}
\toprule
\textbf{$L'$} & \textbf{Construction} & \textbf{Number of orderings} \\
\midrule
$4$  & full enumeration of $4!$ orederings        & $24$ \\
$8$  & random sampling          & $64$ \\
$16$ & random sampling          & $128$ \\
\bottomrule
\end{tabular}
\end{table}

To cover the MDM regime, the same set of orderings is evaluated under the multi-step parallel-generation schedule with $\mathrm{NFE}\in\{1,2,4,\dots,L'\}$, where $\mathrm{NFE}$ is the number of forward passes through the MDM per block. Smaller $\mathrm{NFE}$ corresponds to coarser parallel denoising. The validation-set size used per evaluation is $100$ batches with the eval batch size of the BD3-LM repo, giving roughly $4\!\times\!10^{5}$ per-block scores per ordering at $L'=4$.

\subsection{Estimator hyperparameters}
\label{app:estimator_hyperparameters}

All estimators in Tables~\ref{tab:owt_ppl_main}--\ref{tab:lm1b_ppl_main} share a common per-row number of orderings depending on $L'$. This number is split between the two halves used by TUBE, and the corresponding hyperparameters of CUBO, TVO$_U$, and IS-VG-B are chosen so that all estimators draw orderings per row. Table~\ref{tab:reproduction_estimator_hp} summarizes the configuration.

\begin{table}[h]
\centering
\small
\caption{\textbf{Estimator hyperparameters used in the main tables.} The per-row number of orderings is fixed per block size so that estimators are compute-comparable across rows. ELBO$_K$ uses all available orderings deterministically.}
\label{tab:reproduction_estimator_hp}
\renewcommand{\arraystretch}{1.15}
\setlength{\tabcolsep}{6pt}
\begin{tabular}{@{} l c c c @{}}
\toprule
\textbf{Estimator} & \textbf{$L'=4$} & \textbf{$L'=8$} & \textbf{$L'=16$} \\
\midrule
TUBE (split $B/2 + B/2$)                                  & $12+12$    & $32+32$    & $64+64$ \\
CUBO$_{\beta=2}$ (single MC set of size $B$)              & $24$     & $64$       & $128$ \\
TVO$_U$ (Riemann partitions $\Lambda$, $K=B$)             & $\Lambda=200$ & $\Lambda=200$ & $\Lambda=200$ \\
IS-VG-B ($n_{\mathrm{p}}=2$, $k=B/4$)                  & $2\times 6$  & $2\times 16$  & $2\times 32$ \\
ELBO$_K$ (full bank, deterministic)                        & $M=24$   & $M=64$    & $M=128$ \\
ELBO (15 random-seed runs)                                 & 15 seeds & 15 seeds  & 15 seeds \\
\midrule
MC re-seeds for $\pm$ std                                  & $10$ & $10$ & $10$ \\
\bottomrule
\end{tabular}
\end{table}

The heatmap in Appendix \ref{app:estimator_comparison} sweeps wider hyperparameter ranges $\beta\in\{1.0,1.5,2.0,3.0,5.0\}$ for CUBO. The $\pm$ standard deviations in the tables are computed across $10$ MC re-seeds, each drawing a fresh per-row subset of size $B$ from the bank.

\subsection{Compute}
\label{app:compute}

All evaluations run on a single NVIDIA H200 140\,GB GPU under PyTorch 2.5.1 / CUDA 12.1. Wall-clock times for the OWT validation split with $100$ batches are summarized in Table~\ref{tab:reproduction_compute}. LM1B times are within $30\%$ of OWT. Total evaluation across both datasets and all three block sizes is roughly $50$--$60$ GPU-hours.

\begin{table}[h]
\centering
\small
\caption{\textbf{Wall-clock per evaluation on a single H200 140\,GB} (OWT/LM1B, $100$ batches, default eval batch size).}
\label{tab:reproduction_compute}
\renewcommand{\arraystretch}{1.15}
\setlength{\tabcolsep}{6pt}
\begin{tabular}{@{} l c c c @{}}
\toprule
\textbf{Evaluation} & \textbf{$L'=4$} & \textbf{$L'=8$} & \textbf{$L'=16$} \\
\midrule
ARM baseline (per-block AR)                                & $12$\,m  & $12$\,m  & $12$\,m \\
ARM-FT (synth-finetuned, OWT only)                          & $14$\,m  & $14$\,m  & $14$\,m \\
ELBO ($15$ seeds)                                          & $28$\,m  & $32$\,m  & $38$\,m \\
Ordering bank ($M\in\{24,64,128\}$ orderings)            & $35$\,m  & $1$\,h\,$25$\,m  & $2$\,h\,$50$\,m \\
Multi-step bank (full $\mathrm{NFE}$ schedule)             & $42$\,m  & $1$\,h\,$40$\,m  & $3$\,h\,$30$\,m \\
\bottomrule
\end{tabular}
\end{table}


\newpage

\end{document}